\documentclass[sigconf]{acmart}
\usepackage{caption}
\usepackage{subcaption}
\usepackage{algorithm}
\usepackage{algpseudocode}
\usepackage{tikz}
\usetikzlibrary{positioning,arrows.meta}
\usepackage{mathtools}
\DeclarePairedDelimiter\ceil{\lceil}{\rceil}
\DeclarePairedDelimiter\floor{\lfloor}{\rfloor}
\AtBeginDocument{%
  \providecommand\BibTeX{{%
    \normalfont B\kern-0.5em{\scshape i\kern-0.25em b}\kern-0.8em\TeX}}}

\copyrightyear{2024}
\acmYear{2024}
\setcopyright{acmlicensed}\acmConference[GECCO '24 Companion]{Genetic and Evolutionary Computation Conference}{July 14--18, 2024}{Melbourne, VIC, Australia}
\acmBooktitle{Genetic and Evolutionary Computation Conference (GECCO '24 Companion), July 14--18, 2024, Melbourne, VIC, Australia}
\acmDOI{10.1145/3638530.3664158}
\acmISBN{979-8-4007-0495-6/24/07}

\begin{document}

\title{Analyzing the Runtime of the Gene-pool Optimal Mixing Evolutionary Algorithm (GOMEA) on the Concatenated Trap Function}

\author{Yukai Qiao}
\authornote{Both authors contributed equally to this research.}

\orcid{0000-0002-9876-4353}
\affiliation{
    \institution{School of Information Technology and Electrical Engineering\\
      University of Queensland}
      \streetaddress{1 Th{\o}rv{\"a}ld Circle}
      \city{Brisbane}
      \country{Australia}
}
\email{kai.barnes@uqconnect.com}

\author{Marcus Gallagher}
\orcid{0000-0002-6694-9572}
\affiliation{
  \institution{School of Information Technology and Electrical Engineering\\
  University of Queensland}
  \streetaddress{1 Th{\o}rv{\"a}ld Circle}
  \city{Brisbane}
  \country{Australia}}
\email{marcusg@uq.edu.au}

\renewcommand{\shortauthors}{Yukai and Marcus}

\begin{abstract}
The Gene-pool Optimal Mixing Evolutionary Algorithm (GOMEA) is a state of the art evolutionary algorithm that leverages linkage learning to efficiently exploit problem structure. By identifying and preserving important building blocks during variation, GOMEA has shown promising performance on various optimization problems. In this paper, we provide the first  runtime analysis of GOMEA on the concatenated trap function, a challenging benchmark problem that consists of multiple deceptive subfunctions. We derived an upper bound on the expected runtime of GOMEA with a truthful linkage model, showing that it can solve the problem in $O(m^{3}2^k)$ with high probability, where $m$ is the number of subfunctions and $k$ is the subfunction length. This is a significant speedup compared to the (1+1) EA, which requires $O(ln{(m)}(mk)^{k})$
expected evaluations. 

\end{abstract}




\maketitle

\section{Introduction}

Evolutionary Algorithms (EAs) are a class of optimization algorithms inspired by the principles of natural evolution. They maintain a population of candidate solutions and apply selection and variation operators to improve the solutions over generations. The performance of EAs heavily depends on the choice of variation operators and the problem structure. Recombination operators, such as crossover, play a crucial role in exploiting the structure of the problem by combining partial solutions from different individuals. However, the effectiveness of recombination is often hindered by the linkage between problem variables, i.e., the inter-dependencies between different components of the solution.

To address this issue, the Gene-pool Optimal Mixing Evolutionary Algorithm (GOMEA)\cite{thierens2011optimal} was introduced. GOMEA is a novel EA that employs a linkage model to capture the dependencies between problem variables and performs variation by mixing partial solutions according to this model. The linkage model is represented by a Family of Subsets (FOS), which is a set of subsets of problem variables that are considered to be interdependent. By using this linkage knowledge, GOMEA aims to efficiently exploit the problem structure and avoid disrupting important building blocks during variation.

In this paper, we perform the first rigorous runtime analysis of GOMEA on the concatenated trap function, a well-known benchmark problem that consists of multiple deceptive subfunctions. We derive an upper bound on the expected runtime of GOMEA with a truthful FOS, showing that it can efficiently solve the problem by recombining optimal subsolutions. We also provide an upper bound on the simple (1+1) EA, demonstrating the significant speedup achieved by GOMEA's informed recombination.

The rest of the paper is organized as follows. Section 2 provides the necessary definitions and background on the concatenated trap function, the (1+1) EA, and GOMEA. In Section 3, we present our theoretical analysis of the runtime of both algorithms on the concatenated trap function. Section 4 reports our experimental results, verifying the derived upper bound for GOMEA. Finally, Section 5 concludes the paper and discusses potential future work.

\section{Definitions}
\subsection{Concatenated Trap Function}

The concatenated trap function \cite{deb1993analyzing} \cite{thierens2010linkage} is a concatenation of $m$ of $k$-bit trap functions, each designed to have a local optimum and a global optimum. The function is defined as follows:

Let $x = (0,1)^{mk}$ be a binary string of length $mk$. The string is divided into $m$ non-overlapping substrings of equal length $k$.

The fitness of each substring $x^{(i)}$ is evaluated using the trap function $f_{trap}$:

\begin{equation}
f_{trap(k)}(x^{(i)}) = \begin{cases}
k - u(x^{(i)}), & \text{if } u(x^{(i)}) < k \\
k, & \text{otherwise}
\end{cases}
\end{equation}

where $u(x^{(i)})$ is the number of ones in the substring $x^{(i)}$.

The overall fitness of the candidate solution $x$ is the sum of the fitness values of all the substrings:

\begin{equation}
f_{MK}(x) = \sum_{i=1}^{m} f_{trap(k)}(x^{(i)})
\end{equation}

This is an additively separable function, with each subfunction being NP hard with growing $k$. For each of the trap function, a hill climber is expected to be either randomly initialized at the global optimum, or to be at the slope leading to the local optimum. Once trapped at the local optimum, the only way to get out is by performing a $k$ bits jump, which is exponentially unlikely with growing $k$.

\subsection{(1+1) Evolutionary Algorithm}

The (1+1) Evolutionary Algorithm, denoted as (1+1) EA, is a simple yet fundamental evolutionary algorithm. It maintains a population of size one and applies mutation as the sole variation operator, as described in Algorithm 1.

\subsection{Family Of Subsets}
As a form of substructure representation, a Family Of Subsets (FOS)\cite{thierens2011optimal} $\mathcal F$ is a set of subsets of $\mathcal S =\{0,1,2,..,l-1\}$ which are the indices of problem parameters with a size of $l$. Each element $\mathcal F_{i}$ is assumed to be relatively independent of the rest of the parameters $\mathcal S - \mathcal F_{i}$.

In a Marginal Product FOS, each element is mutually exclusive of the other; $\mathcal F_{i}\cap\mathcal F_{j}=\emptyset $ for any $\mathcal F_{i},\mathcal F_{j}\in\mathcal F$. In its simplest form, each subset consists of only one parameter which is the univariate FOS.

For the purpose of this paper, we only consider a separable problem that can be truthfully represented by a Marginal Product FOS. We say a MP FOS is truthful if each subset contains only inputs of a specific subfunction. For example, a truthful MP FOS for the concatenated trap function with parameter $m$ and $k$ would be: $\{\{ik, ik+1, ... ik+k-1\}|k\in[m-1]\}$.

\begin{algorithm}
\caption{(1+1) EA}
\begin{algorithmic}[1]
\State Initialize a candidate solution $x \in \{0, 1\}^n$
\While{termination criterion not met}
\State $x' \gets mutate(x, p)$ //Each bit has an independent probability p of flipping
\If{$f(x') \geq f(x)$}
\State $x \gets x'$
\EndIf
\EndWhile
\State \textbf{return} $x$
\end{algorithmic}
\end{algorithm}

\begin{algorithm}
\caption{GOMEA with random selection}\label{alg:om}
\begin{algorithmic}
\Function{GOM}{$p_{0},P, \mathcal F$}
\For{$\mathcal F_{i}$ in $\mathcal F$}
    \State $p_{1} \gets randomSelection(P-p_{0})$
    \State $p_{0}^{new} \gets crossWithMask(p_{0},p_{1},\mathcal F_{i})$
    \State $evaluateFitness(p_{0}^{new})$
    \If{$p_{0}^{new}.fitness > p_{0}.fitness$}
        \State $p_{0} \gets p_{0}^{new}$
    \EndIf
\EndFor
\State Return $p_{0}$
\EndFunction
\\
\Function{GOMEA}{$\mathcal F$}
\While{termination criterion not met}
\State $P \gets uniformInitialization$
\State $p_{0} \gets uniformSelection(P)$
\State $p_{0} \gets GOM(p_{0},P,\mathcal F)$
\EndWhile
\State Return $p^{*}$ // Return the individual with highest fitness
\EndFunction
\end{algorithmic}
\end{algorithm}

\subsection{Gene-pool Optimal Mixing EA}
In an Optimal Mixing (OM) operator the FOS is traversed in a random order and an offspring is produced by two parents swapping genetic
material using a different FOS element as a crossover mask. Specifically, the FOS element describes a subset of variables to be swapped
into the first parent from the second parent. The offspring is only
accepted if it outperforms the first parent and when it does it will
replace the first parent. A pseudo-code is provided in Algorithm.2.

Figure \ref{fig:GOM_example1} and Figure \ref{fig:GOM_example2} provide two examples of a GOM process with or without a truthful FOS on the concatenated trap function. It can be clear that with a truthful FOS, an individual will never lose an optimal sub solution on a fully separable problem, conversely, without a truthful FOS, one risks of trading an optimal sub solutions with inferior ones but with higher short term fitness gain.

\begin{figure}
\centering
\begin{tikzpicture}[node distance=2cm]
\node[text width = 5cm] (fos) {$\mathcal F=\{\{1,2,3\}, \{4,5,6\}, \{7,8,9\}\}$};
\node[below = 2pt of fos, text width = 5cm] (d11) {Donor 1: \textbf{111} 011 011 fitness = 3};
\node[below = 2pt of d11, text width = 5cm] (d21) {Donor 2: \textbf{000} 000 110 };
\node[right = 20pt of d11] {$\mathcal F_{i}=\{1,2,3\}$};
\node[below = 2pt of d21, text width = 5cm] (o1) {Offspring: \textbf{000} 010 011 fitness = 2};
\node[below right = 0pt and -75pt of o1] {discarded};

\node[below = 15pt of o1, text width = 5cm] (d12) {Donor 1: 111 \textbf{011} 011 fitness = 3};
\node[below = 2pt of d12, text width = 5cm] (d22) {Donor 3: 101 \textbf{111} 001 };
\node[right = 20pt of d12] {$\mathcal F_{i}=\{4,5,6\}$};
\node[below = 5pt of d22, text width = 5cm] (o2) {Offspring: 111 \textbf{111} 011 fitness = 6};
\node[below right = 0pt and -75pt of o2] {accepted};

\node[below = 15pt of o2, text width = 5cm] (d13) {Donor 1: 111 111 \textbf{011} fitness = 6};
\node[below = 2pt of d13, text width = 5cm] (d23) {Donor 4: 010 111 \textbf{000} };
\node[right = 20pt of d13] {$\mathcal F_{i}=\{7,8,9\}$};
\node[below right = 5pt and -150pt of d23, text width = 6cm] (o3) {Offspring: 111 111 \textbf{000} fitness = 8, accepted};

\draw[->]
    (o1) -- (d12);

\draw[->]
    (o2) -- (d13);
\end{tikzpicture}
\caption{A GOM process with a truthful MP FOS on the concatenated trap function with k=3, m=3}
\label{fig:GOM_example1}
\end{figure}

\begin{figure}
\centering
\begin{tikzpicture}[node distance=2cm]
\node[text width = 5cm] (fos) {$\mathcal F=\{\{1,2,3,4,5,6\}, \{7,8,9\}\}$};
\node[below = 2pt of fos, text width = 5cm] (d11) {Donor 1: \textbf{111 011} 011 fitness = 3};
\node[below = 2pt of d11, text width = 5cm] (d21) {Donor 2: \textbf{000 000} 110 };
\node[right = 20pt of d11] {$\mathcal F_{i}=\{1,2,3,4,5,6\}$};
\node[below = 2pt of d21, text width = 5cm] (o1) {Offspring: \textbf{000 000} 011 fitness = 4};
\node[below right = 0pt and -75pt of o1] {accepted};

\node[below = 15pt of o1, text width = 5cm] (d12) {Donor 1: 000 000 \textbf{011} fitness = 4};
\node[below = 2pt of d12, text width = 5cm] (d22) {Donor 4: 010 111 \textbf{000} };
\node[right = 20pt of d12] {$\mathcal F_{i}=\{7,8,9\}$};
\node[below right = 5pt and -150pt of d22, text width = 6cm] (o2) {Offspring: 000 000 \textbf{000} fitness = 6, accepted};

\draw[->]
    (o1) -- (d12);

\end{tikzpicture}

\caption{A GOM process with an untruthful MP FOS on the concatenated trap function with k=3, m=3}
\label{fig:GOM_example2}
\end{figure}
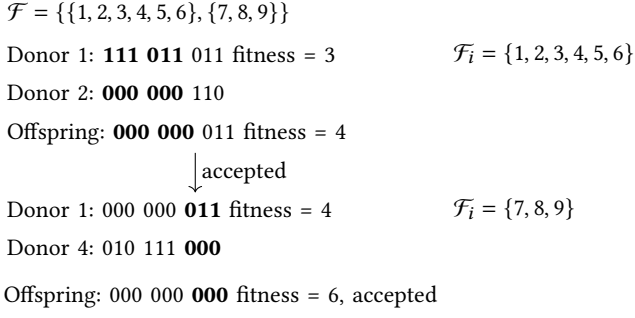

\section{Analysis}
\subsection{(1+1) EA runtime on the Concatenated Trap Function}

\begin{theorem}
\label{thm:ea_runtime}
The expected optimization time of (1+1)EA on concatenated trap function with $m$ substrings of length $k$ is  $O(\ln{m}(mk)^k)$ with a mutation rate of $\frac{1}{mk}$
\end{theorem}
\begin{proof}
First, we separate the search process into two stages. In stage one, at least one substring is neither at the local optimum nor the global optimum of the trap function. In stage two, all substrings are either at the local optimum or the global optimum and the search finishes when all substrings are at the global optimum. In the first stage, it only requires a single 1 being flipped to 0 for an increase of fitness, this is then reduced to a OneMax problem with a problem size of $m(k-1)$ and mutation probability $\frac{1}{mk}$ since each local optimum provides a fitness of $k-1$ at most. Without further proving, the expected time spent in stage 1 is $O(mk\ln{mk})$. For stage two we apply the multiplicative drift theorem\cite{lengler2020drift}. Let state $s$ be the number of non-optimal substrings. To make an improvement, a whole substring needs to be flipped at once, this yields a multiplicative drift of at least:
\begin{equation}
\begin{split}
    \Delta_{t}(s)\geq & s(\frac{1}{mk})^{k}(1-\frac{1}{mk})^{(m-s)k}\\
                \geq & s(\frac{1}{mk})^{k}(1-\frac{1}{mk})^{mk}\\
                \geq & \frac{s}{e}(\frac{1}{mk})^{k} \\
                = & \delta s 
\end{split}
\end{equation}
where $\delta = \frac{1}{e(mk)^{k}}$. This gives the bound:
\begin{equation}
\begin{split}
    O(\frac{1 + \ln{m}}{\delta}) = & O(e(mk)^{k} + e\ln{m}(mk)^{k}) \\
                                = & O(\ln{m}(mk)^k)
\end{split}
\end{equation}
\end{proof}

The runtime in stage 1 is absorbed by stage 2, thus the final upper bound is $O(\ln{m}(mk)^k)$.

\subsection{GOMEA runtime on the Concatenated Trap Function}

The proof consists of two parts. First, we show that the initial population will contain all optimal substrings most of the time for a large enough population size. Then, assuming a truthful MP FOS is given, we show that GOMEA can recombine the optimal substrings across the population efficiently by approximating this process as a diffusion process.

\begin{lemma}
\label{lem:gomea_pop}
The probability of all $m$ optimal substrings to occur at least once in a population size $\mu=cm2^{k}$ is $1-\epsilon$ where $\epsilon = me^{-cm}$ which is exponentially small with $m$ and $c=\Omega(1)$.
\end{lemma}

\begin{proof}
Let $X_{i}$ denote the number of optimal substrings for the $i$-th trap function in the population. It is clear that $E[X_{i}] = \mu \frac{1}{2^k}$ as the probability of a uniformly initiated bitstring with size of $k$ being all ones is $\frac{1}{2^k}$. Applying the Chernoff bound to the probability that $X_{i} = 0$ i.e. none of the substring at subfunction $i$ is all ones.

\[P(X_{i}=0)\leq e^{\frac{-\mu}{2^{k}}}\]

Using the union bound, the probability that at least one subfunction has no optimal strings for $\mu$ individuals is at most:
\[P \left(\bigcup_{i=1}^{m}P(X_{i}=0) \right) \leq me^{\frac{-\mu}{2^{k}}}\]

The probability for each subfunction optimal string to appear at least once in the population is then at least $1 - me^{\frac{-\mu}{2^{k}}}$. Let $\mu = cm2^k$, the probability becomes $1-me^{-cm}$ with the second term being exponentially small with m.
\end{proof}

\begin{theorem}
\label{thm:gomea_runtime}
    The expected optimization time of GOMEA with a truthful FOS on concatenated trap function with $m$ substrings of length k is $O(m^{3}2^{k})$ with a probability of $1-\epsilon$ where $\epsilon = me^{-cm}$ which is exponentially small with $m$, if the population size $\mu=cm2^{k}$ where $c=\Omega(1)$
\end{theorem}
\begin{proof}
    The property of the Gene Pool Optimal Mixing operator determines that when a trap function optimum is obtained by an individual, it will never be lost from recombination given that the FOS is truthful, since the mixing is done one subset at a time and only an improvement is accepted. This means that there is no hitchhiking where optimal genes are out-competed by numerous non-optimal genes. When a uniform selection is used, the spreading rate of an optimal substring depends solely on its current concentration in the population.This also means the portion of optimal substrings at a timestep for each subfunction are i.i.d. Formally, let $E(P_{t})$ denote the expected portion of individuals containing an optimal substring for a subfunction at the $t$th GOM. At each GOM an offspring is produced using each subset in the FOS as a crossover mask with a random donor each time (i.e. $m$ times for this problem), and the probability of a randomly selected parent gaining an optimal substring is $(1-E(P_{t}))E(P_{t})$. The expected change is then:
    \[E(P_{t+1}) - E(P_{t}) = \frac{(1-E(P_{t}))E(P_{t})}{\mu}\]

    For the purpose of asymptotic analysis, this stochastic process can be estimated by a continuous-time diffusion process where $p_{t}$ denotes the portion of optimal substrings at time step $t$, location $i$:
    \[\frac{p_{t}}{dt}= \frac{(1-p_{t})p_{t}}{\mu}\]
    This is a logistic differential equation, and has a known solution with $p_{0} = \frac{1}{\mu}$:
    \[p_{t} = \frac{1}{1+(\mu + 1)e^{-(\frac{t}{\mu})}}\]
    Let $t=\mu m$
    \[p_{t}=\frac{1}{1+(\mu+1)e^{-m}}\]
    substitute $\mu = cm2^{k}$ in
    \[p_{t} = \frac{1}{1+(cm2^k)e^{-m}}\]
    Assume $m$ grows faster than $k$, i.e. $m>ck$ where $c$ is a constant greater than $\log_{2}^{e}$, $p_{t}$ is exponentially close to 1 as $m$ grows, which also means that the portion of individuals with all optimal substrings $\prod_{1}^{m}p_{t}$ is exponentially close to 1. Thus, it takes $O(m^{2}2^{k})$ GOM operations for the optimal solution of concatenated trap function to take over the population, note that each GOM contains $m$ evaluations, resulting in an take over time of $\Theta(m^{3}2^{k})$. $O(m^{3}2^{k})$ is then an upper bound of the first hitting time of the occurrence of the optimal solution.
        
\end{proof}
This is a significant speed-up compared to the 1+1 EA for $k>3$ since we avoided the expensive $(mk)^{k}$ term, which corresponds to the time complexity of jumping $k$ bits by mutation only. This shows the tremendous advantage that can be achieved by having a diverse population and a recombination operator that preserves optimal subsolutions.

This analysis does not consider the affect of mutation. However, if the mutation is applied after a crossover with an FOS element as mask and is limited to the FOS element only, it does not change the fact that an optimal substring will not be lost during this process i.e. no hitchhiking. The probability of a selected parent gaining an optimal substring will then be multiplied by some constant greater than $\sqrt[m]{\frac{1}{e}}$ as the probability of mutation flips no bits, as a consequence, $t$ will need to be at most $\sqrt[m]{e}$ times larger for $p_{t}$ to converge to 1 at the same speed as no mutation. For a randomly initialized population, mutation is still necessary in case the population does not contain all optimal subsolutions, even though it would mean a $\Omega((mk)^{k})$ runtime in this case. For a very short time window, a less than $k$ bit jump might also result in a gain of optimal substring before the selective pressure eliminate the number of ones in the sub-optimal substrings.

This upper bound can be expanded to any separable function consist of $m$ subfunctions with $k$-bit bitstrings as domain, if replacing any subfunction value to the optimal value yields an improvement in the total fitness. Formally:
\[f_{i}:[0,1]^{k}\mapsto \mathbb{R},f_{i}(\hat{x})\geq f_{i}(x)\text{ for } x\in [0,1]^{k}  ,G:\mathbb{R}^{m} \mapsto \mathbb{R}\]
\[F(x_{1},x_{2},...,x_{m}) = G(f_{1}(x_{1}), f_{2}(x_{2}), ...,f_{m}(x_{m})) \]
\[G(f_{1}(x_{1}), f_{2}(x_{2}), ...,f_{m}(x_{m}) \leq G( ... f_{i}(\hat{x}_{i}), ...,f_{m}(x_{m})) \text{ for } i\in [m] \]

\section{Experiment}

\subsection{Verification of Theorem \ref{thm:gomea_runtime}}
To verify the upper bound we initialize a population of size $m2^k$, with bitstrings length $mk$ with all 0 bits, and for each FOS element, a random individual receives a substring of all ones. This is to simulate the worst case scenario while not considering the case where the population does not contain all the optimal substrings. All results are averaged over 100 runs.

\begin{figure*}[ht]
  \centering
  \begin{tabular}{cc}
  \subfloat[$k=4$]{\includegraphics[width=0.45\linewidth]{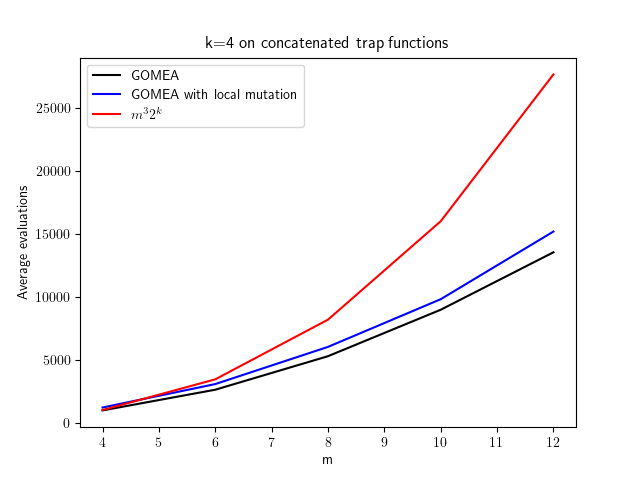}} &
  \subfloat[$k=5$]{\includegraphics[width=0.45\linewidth]{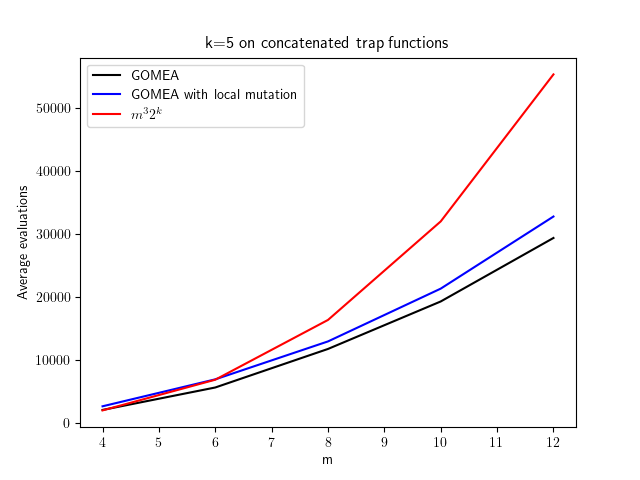}} \\
  \subfloat[$k=6$]{\includegraphics[width=0.45\linewidth]{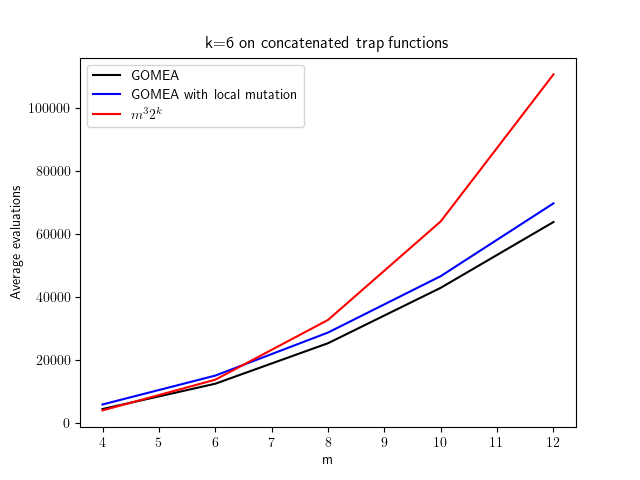}} &
  \subfloat[$k=7$]{\includegraphics[width=0.45\linewidth]{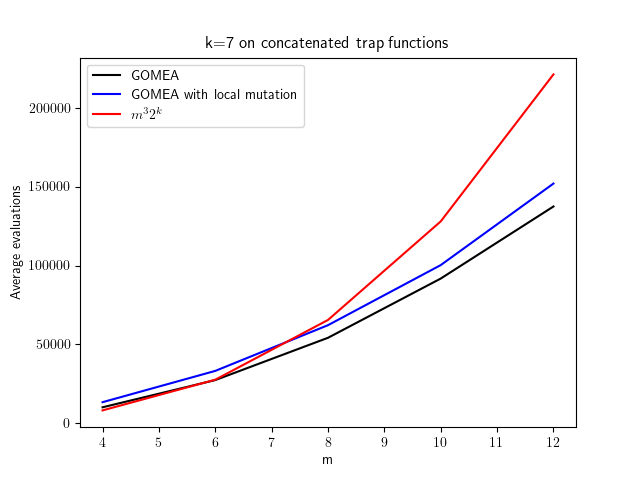}} \\
  \end{tabular}
  \begin{tabular}{c}
  \subfloat[$(1+1) EA$]{\includegraphics[width=0.45\linewidth]{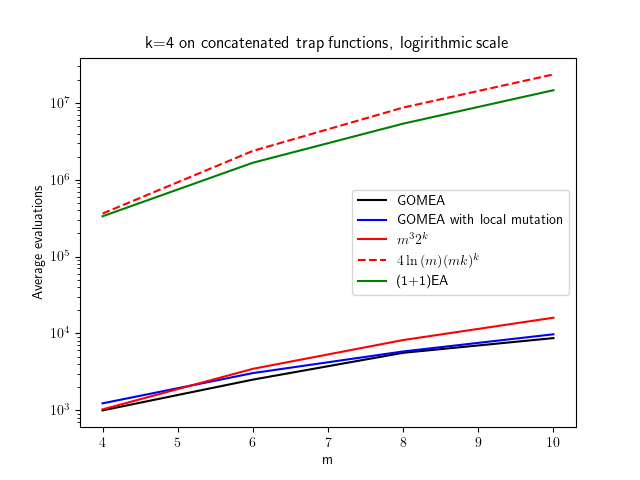}}
  \end{tabular}
  \caption{(a) to (d) Upper bound vs worst case performance (e) Comparison with (1+1) EA with k=4 in logarithmic scale}
  \label{fig:gomea_bound}
\end{figure*}

The results are shown in Figure \ref{fig:gomea_bound}. It is clear that the upper bound is valid in most cases, except for when $m$ is too small compared with $k$, as we have discussed that $m$ needs to outgrow $k$ by a constant for the bound to be accurate. The bound is also tight when $m$ is small, and the gap between the bound and the actual runtime increases as $m$, which is expected, since the bound is an estimation of take over time of the optimal solution rather than the hitting time. The effect of including a local mutation where only the bits in the current FOS element are mutated, agrees with our analysis before, and appears to multiply the runtime by a factor that reduces with $m$. 

\subsection{Success Rate with Constant $c$}

We have shown that the probability of a population containing all optimal substrings is exponentially close to 1 with $c$, when $\mu=cm2^{k}$. To verify this empirically, we run GOMEA 1000 times on the concatenated trap function with $k=4$, $m = 6$. To show the capability of GOM to preserve optimal genes, we also run a ($\mu$+1) GA with deterministic crowding\cite{louis1993genetic}, uniform crossover and uniform selection. GOM and deterministic crowding shares the property where an offspring directly competes with its closest parent, thus it is only for a fair comparison for GA to apply this diversity preserving mechanism. A run is considered successful, if the optimal solution is found within $2cm^{3}k^{2}$. For ($\mu$+1) GA, this limit is increased by 10 times.

The results in Figure \ref{fig:gomea_success} suggest that as $c$ increases, the success rate does approach 1 exponentially with GOMEA. With a local mutation, the success rate is also constantly higher. The reason is as we have discussed, that before the non-optimal substrings are pushed to the local minimum, some optimal strings might be gained by mutation at a much higher probability than $k$-bits jump. Consequentially, the runtime is higher since the optimal substrings gained from mixing might be disrupted by mutation. The ($\mu$+1) GA on the other hand, even with all optimal substrings available in the population, only found optimal solution around half of the time at $c=2$, with 10 times higher runtime limit.

\begin{figure*}[ht]
  \centering
  \begin{tabular}{cc}
  \subfloat[]{\includegraphics[width=0.47\linewidth]{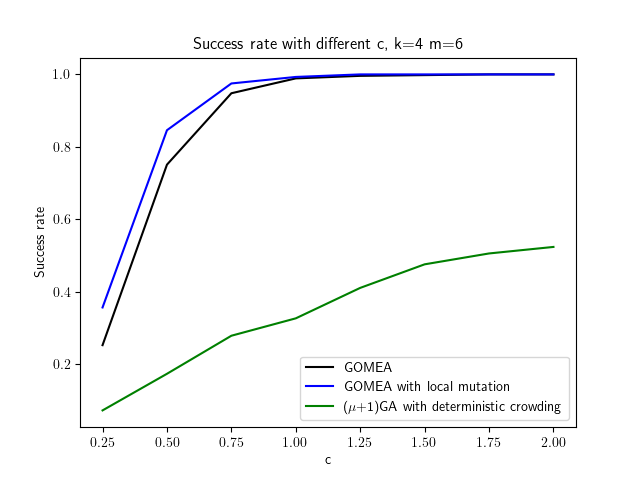}} &
  \subfloat[]{\includegraphics[width=0.47\linewidth]{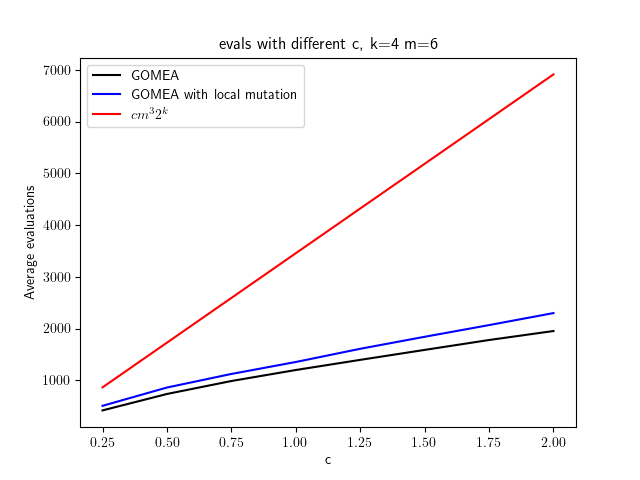}} \\
  \end{tabular}
  
  \begin{tabular}{c}
  \subfloat[]{\includegraphics[width=0.47\linewidth]{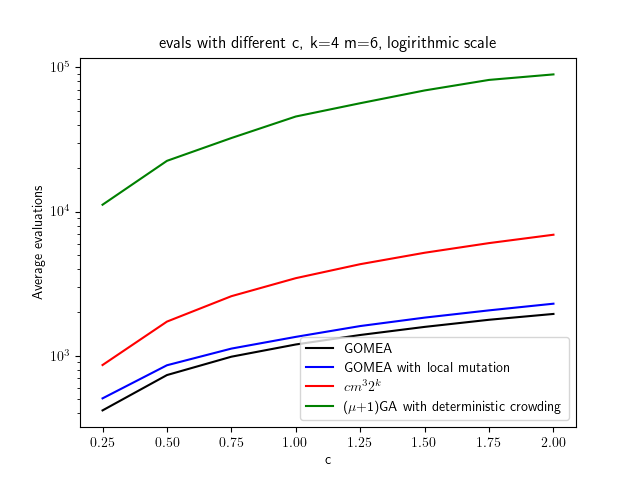}}
  \end{tabular}
  
  \caption{(a) The success rate of GOMEA and ($\mu$+1)EA with deterministic crowding (b) Corresponding average runtime of succeeded runs.$k=4$ ,$m = 6$, over 1000 runs (c) Average runtime including ($\mu$ + 1) GA in logarithmic scale}
  \label{fig:gomea_success}
\end{figure*}

\section{Future Work}
The theoretical and empirical results presented in this paper demonstrate the effectiveness of GOMEA on the concatenated trap function, a problem with a clear modular structure and deceptive subfunctions. Our analysis shows that GOMEA can efficiently exploit this problem structure by recombining optimal subsolutions using a truthful linkage model. This leads to several promising directions for future research.

It would be interesting to investigate the performance of GOMEA on other classes of modular problems, such as the NK-landscape. These problems exhibit a more complex structure, with dependencies between subproblems at different levels of hierarchy. Although empirical results exist, analyzing GOMEA's runtime on these problems could provide further insights into its ability to deal with more challenging problem structures.

Our analysis assumes a truthful linkage model, where the FOS perfectly captures the dependencies between problem variables. In practice, however, the linkage model may not be known a priori and needs to be learned from the population. Analyzing the impact of linkage learning on GOMEA's performance and developing efficient linkage learning techniques for problems with unknown structure is an important direction for future work.

The concatenated trap function is particularly interesting because it represents a class of problems that can be decomposed into NP hard subproblems. The insights we gained on studying this problem thus might potentially transfer into problems sharing similar properties, for example, optimization of a mixture of experts system, which is now a backbone of Large Language Models, or the optimization of neural network weights, which has been shown to have a modular inter-dependency\cite{qiao2023modularity}.

\section{Conclusion}

We presented the first runtime analysis of the Gene-pool Optimal Mixing Evolutionary Algorithm (GOMEA) on the concatenated trap function. By leveraging a truthful linkage model, GOMEA is able to efficiently exploit the modular problem structure and recombine optimal subsolutions. We derived an upper bound of $O(m^3 2^k)$ on the expected runtime, showing a significant speedup compared to the simple (1+1) EA, which requires $O(\ln m(mk)^k)$ evaluations.

Our empirical results verified the derived upper bound and demonstrated GOMEA's ability to find the optimal solution with high probability when the population size is set appropriately. The experiments also highlighted the importance of the Optimal Mixing operator in preserving important building blocks during variation, as evident from the comparison with a $(\mu+1)$ GA using deterministic crowding.

This work provides new theoretical insights into the behavior of GOMEA and its ability to exploit problem structure through linkage learning and optimal mixing. The analysis of the concatenated trap function, a problem consisting of multiple deceptive subfunctions, showcases GOMEA's potential for solving optimization problems with modular interdependencies.

Future research directions include extending the analysis to more complex modular problems, such as NK-landscapes, and investigating the impact of linkage learning on GOMEA's performance when the problem structure is unknown. Furthermore, the insights gained from studying the concatenated trap function could potentially transfer to real-world optimization problems with similar properties, such as the optimization of mixture of experts systems or neural network weights.

In conclusion, this paper contributes to the theoretical understanding of GOMEA and highlights its effectiveness in exploiting problem structure through informed variation operators. The runtime analysis on the concatenated trap function serves as a foundation for further investigations into the behavior and performance of GOMEA on a broader class of optimization problems.

\section{Continuation}

\subsection{Definitions}
\subsubsection{Generalized trap function}
Following the definition in \cite{deb1993analyzing}. A generalized trap function has three extra parameters. $a$ is the local optimal, $b$ is the global optimal, and $z$ defines the starting point of both slopes leading to the local optimal and global optimal.

\begin{equation}
f_{gen\_trap(k)}(x^{(i)}) = \begin{cases}
\frac{a}{z}(z-u(x^{(i))}), & \text{if } u(x^{(i)}) \leq z \\
\frac{b}{k-z}(u(x^{(i))}-z), & \text{otherwise}
\end{cases}
\end{equation}
The trap function we have discussed above is a special case where $a=k-1, b=k, z=k-1$.
\subsubsection{Optimal Region}
We define the optimal region $\Bar{X_{*}}$ of a generalized trap function to be the set of all bitstrings which has a function value greater than the local optimal $a$. 
\begin{equation}
    \Bar{X_{*}} = \{x|f_{gen\_trap(k)}(x)>a\}
\end{equation}

In terms of $u$, this is a region between $[\ceil{\frac{a(k - z)}{b}} + z, k]$

The probability of a bitstring being uniformly initialized in the optimal region is thus $Pr(x\in \Bar{X_{*}})=Pr(Bin(k,0.5)\geq\ceil{\frac{a(k - z)}{b}} + z)$

\subsection{GOMEA with local mutation runtime on concatenated generalized trap function}
The strategy is similar as above. First, we calculate the population size required for each of the optimal region for each subproblem to occur at least once. Then, we show that GOMEA can effectively recombine those optimal regions and perform hill-climbing afterward.

For simplicity, we refer $Pr(x\in \Bar{X_{*}})$ as $p_{*}$

\begin{lemma}
The probability of all $m$ optimal regions to occur at least once in a population size $\mu=\frac{c}{p_{*}}m$ is $1-\epsilon$ where $\epsilon = me^{-cm}$ which is exponentially small with $m$ and $c=\Omega(1)$.
\end{lemma}
\begin{proof}
Same as the proof in Lemma \ref{lem:gomea_pop}, as we only need to replace $\frac{1}{2^k}$ to $p_{*}$
\end{proof}

\begin{theorem}
\label{thm:gomea_bound_2}
    The expected optimization time of GOMEA with a truthful FOS and a local mutation with mutation rate of $\frac{1}{k}$ on concatenated generalized trap function with $m$ substrings of length k with parameters $a,b,z$ is at most $\floor{\frac{(b - a)(k - z)}{b}}(m\ln{m})(ek(1+\frac{c}{p_{*}}) + \frac{ce}{p_{*}}m\ln{m}) + \frac{c}{p_{*}}m^{3}$ which is $O(\frac{1}{p_{*}}m^{3})$ with a probability of $1-\epsilon$ where $\epsilon = me^{-cm}$ which is exponentially small with $m$, if the population size $\mu=\frac{c}{p_{*}}m$ where $c=\Omega(1)$
\end{theorem}

\begin{proof}
    Since the GOM operator will never lose a substring in the optimal region to a substring not in the optimal region, we can safely apply the same strategy in the previous proof of Theorem \ref{thm:gomea_runtime}. We consider $m$ i.i.d optimization process, with a level definition: $l(X) = max(0,max(u(X)) - (\floor{\frac{(b - a)(k - z)}{b}}))$. $l$ tells us how far has the best substring traveled along the slope toward the global optimal.There are $\floor{\frac{(b - a)(k - z)}{b}}$ levels at most and we will estimate the time spent in each level.

    We separate each level to two stages, in stage one, The best substring spread to a constant portion of the population. In stage two, the local mutation operator climb towards the global optimum by adding a 1.

    In the worst case, the population only contain one substring at level $l$, the spreading of such substring with time can be again modeled as a diffusion process where $p_{t}$ denotes the portion of substring at level $l$ at time step $t$. This time, however, the effect of mutation will put a coefficient of approximately $\frac{1}{e}$ to the diffusion rate since the probability of mutation not changing any bit is:$(1-\frac{1}{k})^{k}\approx\frac{1}{e}$ as $k$ approaches infinity.
    \[\frac{p_{t}}{dt}= \frac{(1-p_{t})p_{t}}{e\mu}\]
    With $p_{0} = \frac{1}{\mu}$
    \[p_{t} = \frac{1}{1+(\mu + 1)e^{-(\frac{t}{e\mu})}}\]
    Let $t=e\mu\ln{m}$
    \begin{equation}
    \begin{split}
        p_{t} & = \frac{1}{1+(\mu + 1)e^{-\ln{m}})} \\
              & = \frac{1}{1+\frac{\frac{c}{p_{*}}m + 1}{m}} \\
              & \approx \frac{1}{1+\frac{c}{p_{*}}}\\
    \end{split}
    \end{equation}
    Which is a constant. The time spent in this stage is thus $t= \frac{ce}{p_{*}}m\ln{m}$

    In stage two, with a constant portion of the population at level $l$,
    the probability of a local mutation improve a level $l$ substring is at least:
    \[\frac{1}{1+\frac{c}{p_{*}}}(\frac{1}{k})(1-\frac{1}{k})^{k-1}>\frac{1}{(ek)(1+\frac{c}{p_{*}})}\]
    The time spent in this stage is then at most: $ek(1+\frac{c}{p_{*}})$

    The expected time for the $i^{th}$ of the $m$ processes to reach maximum level is thus $E(T_{i})=\floor{\frac{(b - a)(k - z)}{b}}(ek(1+\frac{c}{p_{*}}) + \frac{ce}{p_{*}}m\ln{m})$

    However, as we are solving for a concatenation of such subproblems, the expected time for optimal subsolutions to occur at least once for all subproblems is $E(\max\limits_{m}(T_{i}))$. Assuming $T_{i}$ follows exponential distribution and given they are i.i.d, $E(\max\limits_{m}(T_{i}))$ can be approximated\cite{david2004order} as $E(T_{i})\ln(m)$. Each GOM operator also consists of $m$ function evaluations. The expected time spent on finding all optimal subsolutions is in the end $\floor{\frac{(b - a)(k - z)}{b}}(m\ln{m})(ek(1+\frac{c}{p_{*}}) + \frac{ce}{p_{*}}m\ln{m})$. This can be reduced to $O(mk\ln{m} + \frac{1}{p_{*}}m^{2}\ln^{2}{m})$

    Finally, GOMEA need to recombine the optimal subsolutions present at this point, using the same argument in the previous proof, the time is bounded by $\mu m^{2} =\frac{c}{p_{*}}m^{3} = O(\frac{1}{p_{*}}m^{3})$. Assuming $m$ grows faster than $k$, the bound of the previous level climbing process is absorbed, thus the final upper bound is $O(\frac{1}{p_{*}}m^{3})$ asymptotically.
\end{proof}

Again, this bound can be generalized to any subproblems with a given probability $p_{*}$ of being initialized in an optimal region, if an hill climber can efficiently optimize to the global optimum from any point within the optimal region. Including a hill climbing process provides a much deeper understanding of how GOMEA works on a much broader set of problems.

\subsection{Experiment}

\subsubsection{Verification of Theorem \ref{thm:gomea_bound_2}}
For convenience, we set $a=1$ and $b=k$, as we only need to change $z$ in order to control the starting point of the optimal region which is just $z+1$. We then run experiments with different settings and different resulting $p_{*}$ on the worst case population, where for each subproblem, a substring with $z+1$ ones occurs only once, with the rest of substrings being zeros.

We compare the worst case performance to the upper bound of the full expected runtime $E(T)$, and only the last term of $E(T)$, $\frac{1}{p_{*}}m^{3}$, excluding the hill climbing process. The results in Figure.\ref{fig:gomea_bound2}(a) to (e) showed that the runtime of the hill climbing is generously overestimated and the last term $\frac{1}{p_{*}}m^{3}$ seems to contribute the majority of the runtime. Figure.\ref{fig:gomea_bound2}(f) supported our hypothesis that the difficulty of the problem to GOMEA is dominated by $p_{*}$, regardless of the actual size of the problem.

\begin{figure*}[ht]
  \centering
  \begin{tabular}{cc}
  \subfloat[$k=4$]{\includegraphics[width=0.45\linewidth]{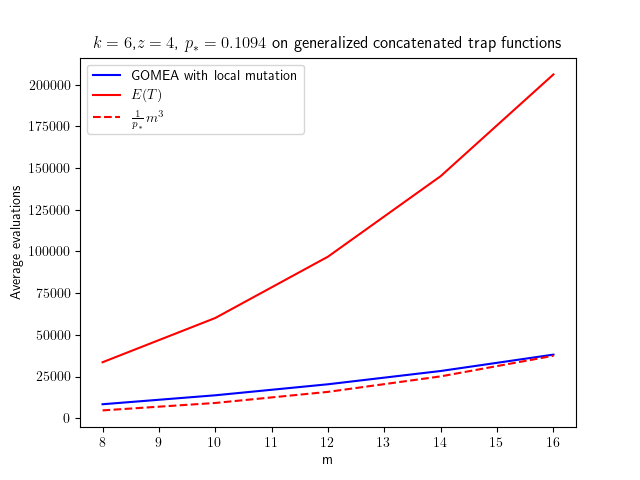}} &
  \subfloat[$k=5$]{\includegraphics[width=0.45\linewidth]{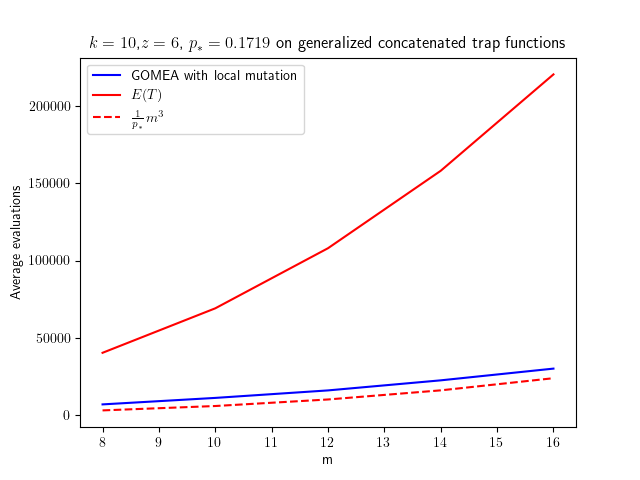}} \\
  \subfloat[$k=6$]{\includegraphics[width=0.45\linewidth]{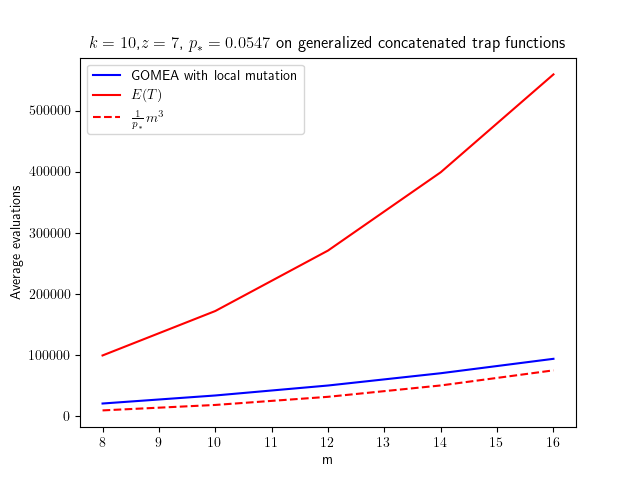}} &
  \subfloat[$k=7$]{\includegraphics[width=0.45\linewidth]{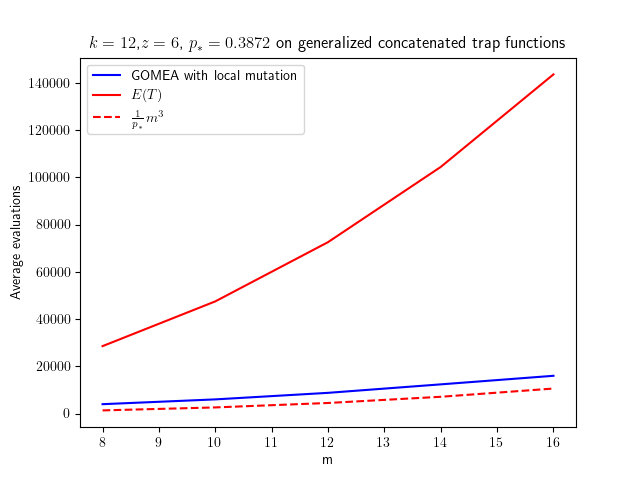}} \\
  \subfloat[$k=4$]{\includegraphics[width=0.45\linewidth]{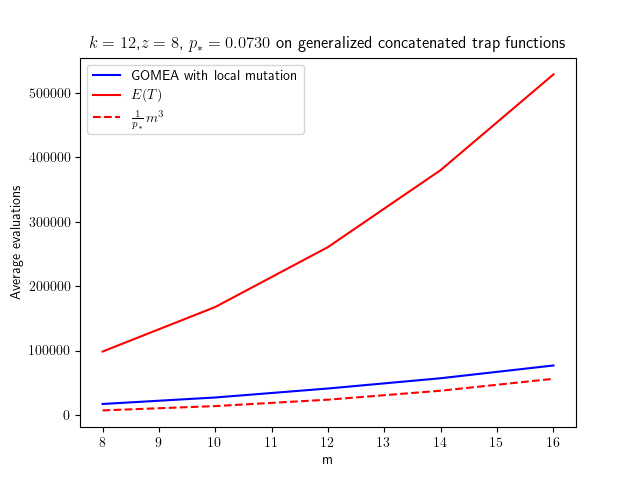}} &
  \subfloat[$k=5$]{\includegraphics[width=0.45\linewidth]{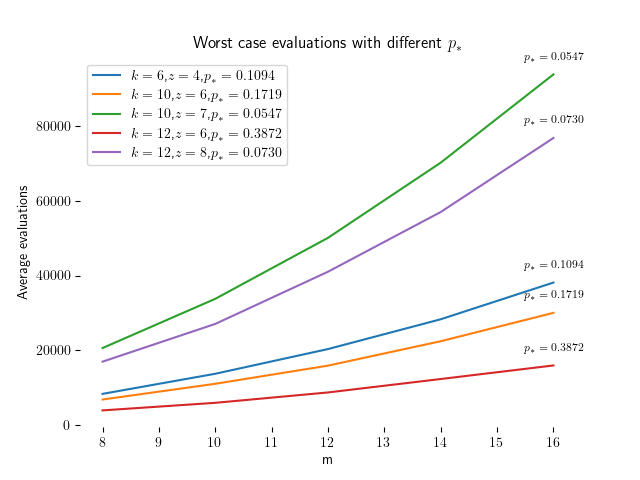}} \\
  \end{tabular}

  \caption{(a) to (e) Upper bound vs worst case performance (e) Comparison between different $p_{*}$}
  \label{fig:gomea_bound2}
\end{figure*}

\subsubsection{Success rate with c and comparison to GA}
A similar experiment is run on GOMEA and GA with deterministic crowding, verifying GOMEA's ability of maintaining successful subsolutions. 1000 runs are performed on the concatenated generalized trap function with $m=8$, $k=6$, $z=4$ with a $p_{*}=0.1094$ consequently. From the previous experiment, we see $\frac{c}{p_{*}}m^{3}$ is actually a pretty accurate estimation rather than the complete $E(T)$, thus the success condition for GOMEA is within 2 times of $\frac{c}{p_{*}}m^{3}$ and 20 times for GA.

The result in Figure.\ref{fig:gomea_success2}(c) and (e) confirmed that GOMEA required a much smaller population size with the concept of optimal region, as we only need substrings in the optimal region rather than be at the global optimal. Surprisingly, GA seems to perform pretty well on this problem as well. This may be because the problem is no longer deceptive, since we keep the local optimum at the value 1, the average schema fitness with more zeroes are going to be lower than the average schema fitness with more ones, resulting in benign hitchhiking, where the crossover favor one good gene over several bad genes. For example, if an existing optimal substring with a fitness value of 6 is to be disrupted during a crossover in GA with deterministic crowding, the offspring will only be accepted if it gains extra fitness value of 6 else where, and it is almost impossible with the local optimum having a low fitness value of 1. In other words, the optimal substrings are maintained even through a disruptive crossover.

We then designed a tailed trap function shown in Fig.\ref{fig:gomea_success2}(b), where the local optimum has a much higher fitness value but the optimal region has the same size. This is however, not a fully deceptive problem since the schemas with the size of $k-1$ are not misleading. With the same $p_{*}$, the performance of GOMEA is identical for these two problems, in fact, the expect runtime would be the same for a hill climber, as the number of bit flips required to escape the local optimum is the same. For GA with deterministic crowding however, the problem is suddenly impossible to solve, unless with extremely lucky initialization.

\begin{figure*}[ht]
  \centering
  \begin{tabular}{c|c}
  \subfloat[]{\includegraphics[width=0.45\linewidth]{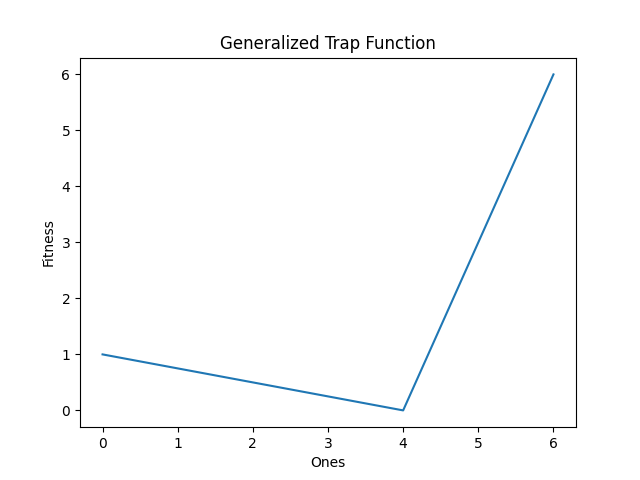}} &
  \subfloat[]{\includegraphics[width=0.45\linewidth]{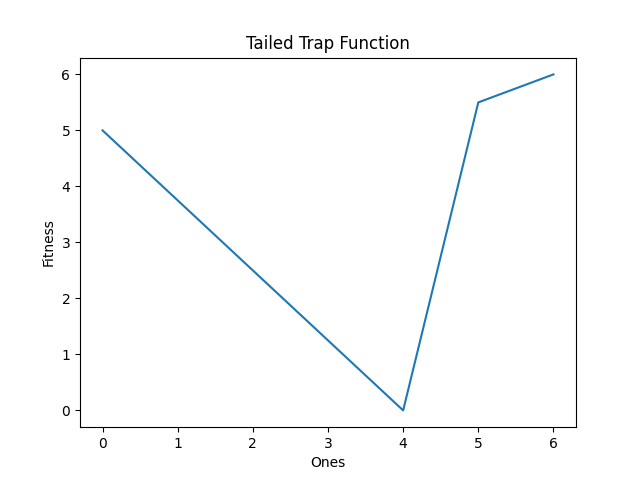}} \\
  \subfloat[]{\includegraphics[width=0.45\linewidth]{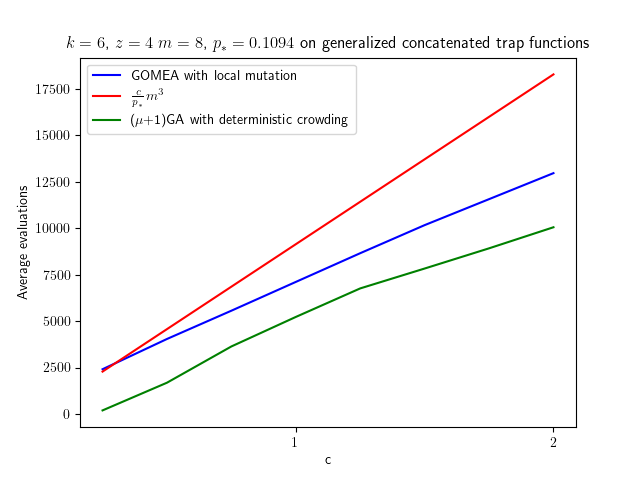}} &
  \subfloat[]{\includegraphics[width=0.45\linewidth]{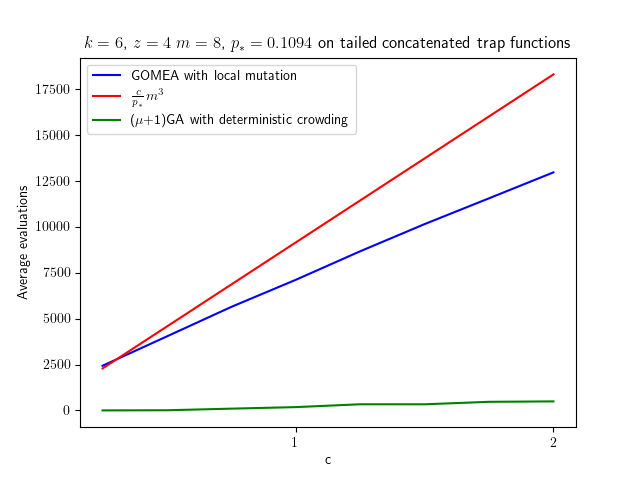}} \\
  \subfloat[]{\includegraphics[width=0.45\linewidth]{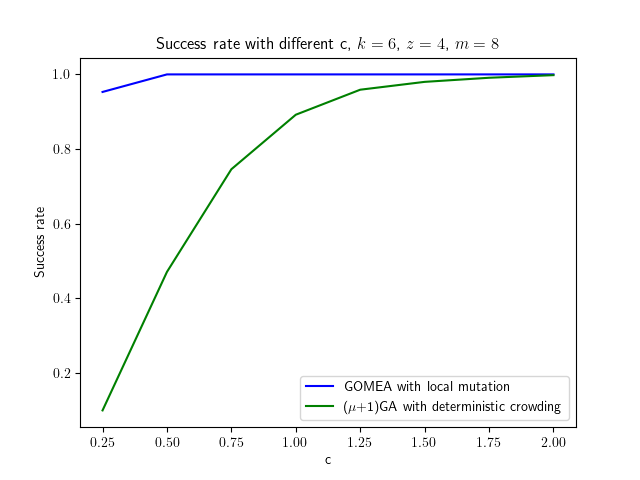}} &
  \subfloat[]{\includegraphics[width=0.45\linewidth]{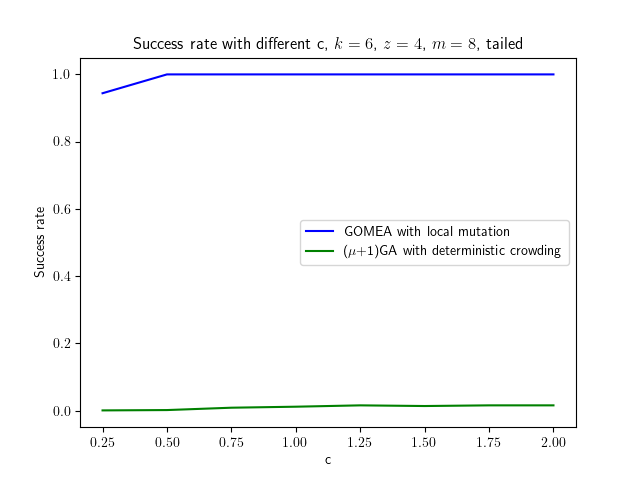}} \\
  \end{tabular}

  \caption{(a)(c)(e) Results of concatenated generalized trap function (b)(d)(f) Results of concatenated tailed trap function}
  \label{fig:gomea_success2}
\end{figure*}

\bibliographystyle{ACM-Reference-Format}
\bibliography{sample-base}

\appendix

\end{document}